\documentclass[letter,11pt]{article}

\usepackage{amsmath}
\usepackage{amsfonts}
\usepackage{amsthm}
\usepackage{mathtools}
\usepackage{textcomp}

\usepackage[]{algorithm2e}

\usepackage[utf8]{inputenc}
\usepackage[english]{babel}
 
\newtheorem{theorem}{Theorem}
\newtheorem{corollary}{Corollary}[theorem]
\newtheorem{lemma}[theorem]{Lemma}
\newtheorem{definition}[theorem]{Definition}

\newcommand{\D}{\mathbf{D}}
\renewcommand{\H}{\mathcal{H}}

\newcommand{\G}{\mathcal{G}}

\newcommand{\X}{\mathcal{X}}
\newcommand{\Y}{\mathcal{Y}}

\newcommand{\E}{\mathop{\mathbb{E}}}
\newcommand{\R}{\mathbb{R}}
\newcommand{\Hi}{\mathbb{H}}
\newcommand{\Ra}{\mathfrak{R}}
\renewcommand{\H}{\mathcal{H}}
\renewcommand{\S}{\mathcal{S}}

\newcommand{\argmin}{\mathop{\mathrm{argmin}}}

\author{Michael Rabadi \\
Spotify \\
New York, NY 10011 \\
\texttt{mrabadi@spotify.com} \\
}

\title{Generalization bound for kernel similarity learning}

\begin{document}

\maketitle

\begin{abstract}%
Similarity learning has received a large amount of interest and is an important tool for many scientific and industrial applications. In this framework, we wish to infer the distance (similarity) between points with respect to an arbitrary distance function $d$. Here, we formulate the problem as a regression from a feature space $\X$ to an arbitrary vector space $\Y$, where the Euclidean distance is proportional to $d$. We then give Rademacher complexity bounds on the generalization error. We find that with high probability, the complexity is bounded by the maximum of the radius of $\X$ and the radius of $\Y$. 
\end{abstract}

\section{Introduction} \label{sec:introduction}

There are many applications where scientists wish to infer similarity between samples. In this framework, we are given a sample $S = (x_1, x_2, ..., x_m), x \in \R^N$ of points with associated features in some vector space $\X \subset \R^N$ and a corresponding distance matrix $\D$, where $\D_{ij} = d(x_i, x_j)$, for some arbitrary distance function (similarity) $d$. The goal is to learn a mapping $h : \X \to \Y$, where the Euclidean distance in $\Y$ is close to the corresponding distances of $\D$. Of course it is not sufficient to simply reconstruct $\D$, since the goal is to infer the distance between new points. Indeed, this framework is completely general and immediately relevant to many scientific and industrial domains.

One notable application comes from neuroscience, where scientists may simultaneously record neural signals and behavioral measurements. One concrete example comes from object recognition research, where a researcher may want to relate behavior (confusion rate between pairs of objects) and neural activity for different objects \cite{ dicarlo2007, hung2005, kriegeskorte2012}. In many cases, neuroscientists will use a linear classifier to infer separability of the neural signal, which implies that a decision making part of the brain could linearly categorize stimuli based on those neural signals (see \cite{dicarlo2007} for review). Unfortunately, this approach does not capture the full behavioral profile and reduces object recognition to a binary behavior. Another approach is to use dimensionality reduction algorithms to infer categorical structure in the neural signals and relate this to behavior \cite{kiani2007, kriegeskorte2008}. Unfortunately, this approach does not generalize to new data, since all of the pairwise neural comparisons and pairwise behavioral comparisons must be measured. Thus, introducing a new object requires comparing to all $m$ previously measured objects. In many cases, this is impractical. Similarity learning is a natural solution to all of these drawbacks. However, we must first transform a behavioral confusion matrix into a distance matrix. We can do this simply with the following equation, $\D_{ij} = 1 - \mathbf{C}_{ij}$, where $\mathbf{C}$ is the confusion matrix. Thus, $\D$ implies a bounded space $\Y$ with radius $\beta = 0.5$. Indeed, learning the distances between points gives a finer measure of separation in neural signals than simply demonstrating the ability for a binary classifier to classify objects based on the corresponding neural signal. Furthermore, we can infer distances between new objects based on their neural signature and guarantee the total error with high probability. This allows us to precisely decide the number of objects to compare such that we can infer the rest of the behavior with the learned hypothesis $h$.

Another important example comes from recommendation systems. Suppose, for example, that we have a global encoding of products in some vector space $\X$, where the distance between two products implies some relationship. We then want to recommend new products to a customer based on their purchase history $S$. One approach is to simply use the distances between the products in $\X$, but there is no reason to believe that global similarity between products is meaningful for a customer. Instead, we might want to construct a new space $\Y$ where the distance between products is proportional to a user's tastes. In this case, we might construct a similarity measure $d$ that is related to a user's purchase or browsing behavior. We can then recommend products by selecting new ones that are similar to a user's purchase history in $\Y$.

Similarity learning is a general framework with diverse applications, however there has been little work on the generalization analysis of the problem. Previously, \cite{cao2016} derived a generalization bound for the case where $\Y$ is a linear transformation of $\X$. Their bound was in terms of the norm of the transformation matrix $\mathbf{A}$ and a term related to the maximum distance between two points in $\X$. They also gave explicit bounds for common norms. Unfortunately their analysis was limited to linear transformations of $\X$.

Here, we give new Rademacher complexity bounds for the similarity learning problem. We give a different formulation of the objective function for learning than \cite{cao2016}, which is convex for certain hypothesis classes $\H$. We then give explicit bounds for both linear and kernel solutions to the optimization problem. We find that, in general, model complexity grows polynomially with the radius of the space. We note, however, that our approach could be used for neural networks as well, since the Rademacher complexity is general. However, we do not yet have the statistical tools to give meaningful upper-bounds on the Rademacher complexity for non-convex hypothesis classes. Therefore, while we will stick with kernel methods in this paper, our research does not preclude the use of more complex hypothesis classes.

\subsection{Preliminaries} \label{sec:prelim}

One simple approach for solving the similarity learning problem is to formulate it as a regression problem. In this case, however, we are not interested in the values of the output space. Instead, we only care about the distance between those points. Therefore, we might update the weights of the regression with respect to the following optimization problem:

$$
\argmin_{\theta} \frac{1}{m^2} \mathbf{1}^\top ( (\hat{\D}_\theta - \D) \otimes (\hat{\D}_\theta - \D) ) \mathbf{1}
$$

Where $\otimes$ denotes the Hadamard product and $\hat{\D}_\theta$ is the distance matrix between the values computed with our learned parameterized mapping $h$. In the case of a kernel mapping, we can regularize with the weight matrix (see Corollaries \ref{corollary:linear} and \ref{corollary:kernel} for a theoretical justification).

$$
\argmin_{\theta} \frac{1}{m^2} \mathbf{1}^\top ( (\hat{\D}_\theta - \D) \otimes (\hat{\D}_\theta - \D) ) \mathbf{1} + \lambda \|w\|_2
$$

where $\lambda$ is the regularization parameter. Also note that, in general, we don't care about the dimensionality of the transformed space $\Y$, and so a simple way of regularizing would be to lower its dimensionality. Indeed, this is an argument for embedding dimensionality reduction into the learning problem. However, as we will see in Corollary \ref{corollary:kernel}, the model complexity does not depend on the dimensionality of the data.

As with most learning problems, we are not interested in minimizing the empirical error of a hypothesis $\hat{R}(h)$. Thus, we wish to prove probabilistic bounds on the generalization error $R(h)$ for some hypothesis $h \in \H$. Here, we will assume a fixed, but unknown distance function $d$. We further assume that data points $x \in \X$ are drawn i.i.d. according to a fixed, but unknown distribution $D$.

Generalization error is typically given as a function of the empirical error $\hat{R}(h)$ plus some complexity of the corresponding hypothesis class. A commonly used complexity is the Rademacher complexity $\Ra_m(\H)$, which is known to have particularly nice, data-dependent properties \cite{bartlett2002}. Formally, the Rademacher complexity is defined as follows: 

\begin{definition}[Rademacher complexity \cite{mohri2012}]
\label{def:rademacher}
Let $S = (z_1, ..., z_m)$ be a fixed sample drawn according to some fixed, but arbitrary distribution $D$. Then for any sample of size $m > 1$, the \textit{Rademacher complexity} of a family of functions $\G$ is defined as follows:

$$
\Ra_m(\G) = \E_{S, \sigma} \bigg[ \sup_{g \in \G} \frac{1}{m} \sum_{i=1}^m \sigma_i g(z_i) \bigg]
$$

Where $\sigma = (\sigma_1, ..., \sigma_m)$ with $\sigma_i$ being a Rademacher random variable, taking values uniformly in $\{-1, +1\}$.
\end{definition}

Indeed, one can think of the Rademacher complexity as the ability of a hypothesis class to correlate with random noise \cite{mohri2012}. Unfortunately, $\Ra_m(\H)$ is not typically accessible. However, we can usually get around this problem by deriving upper bounds, which gives us an explicit bound on the generalization error, which we will do in the next section.

\section{Theoretical analysis} \label{sec:theory}

In this section we give a Rademacher complexity bound for the similarity learning problem. As previously discussed, we wish to minimize the following loss $g = \frac{1}{m^2} \mathbf{1}^\top ( (\hat{\D} - \D) \otimes (\hat{\D} - \D) ) \mathbf{1}$, which is the squared error between the predicted distance matrix $\hat{\D}$ and the true distance matrix $\D$. 

We wish to bound the generalization error for some hypothesis $h$ in a hypothesis class $\H$. To do this, we will need to make use of the following lemma:

\begin{lemma}
\label{lemma:martingale}
Let $S$ and $S^\prime$ be two samples that differ by a single point. We will define $\varphi(S) = \sup_{g \in \G} \E[g] - \hat{\E}_S[g]$. Furthermore, assume that $(\|h(x_i) - h(x_j)\|_2 - \D_{ij})^2 \leq M^2$. Then,

$$
|\varphi(S) - \varphi(S^\prime)| \leq \frac{2M}{m}
$$
\end{lemma}

\begin{proof}

\begin{align*}
| \varphi(S^\prime) - \varphi(S) | &= \sup_{g \in \H} \hat{\E}_S [g] - \hat{\E}_{S^\prime}[g] | \\
&= \sup_{g \in \H} \frac{ \mathbf{1}^\top ( (\hat{\D} - \D) \otimes (\hat{\D} - \D) ) \mathbf{1} - \mathbf{1}^\top ( (\hat{\D}^\prime - \D^\prime) \otimes (\hat{\D}^\prime - \D^\prime) ) \mathbf{1}}{m^2} \leq \frac{2M}{m},
\end{align*}

where we use the fact that changing a single point will change at most $2m$ distances by at most $M$, by assumption.
\end{proof}

\begin{lemma}
\label{lemma:rademacher}
Define $\varphi(S)$ as above and let $\G$ be the class of losses over the hypothesis class $\H$. Then,

$$
\E_S[\varphi(S)] \leq 2 \Ra_m(\G)
$$
\end{lemma}

\begin{proof}
\begin{align*}
\E_S[\varphi(S)] &= \E_S \bigg[ \sup_{g \in \H} \E[g] - \hat{E}_S[g] \bigg] \\
&= \E_{S} \bigg[ \sup_{g \in \H}  \E_{S^\prime}\big[ \hat{\E}_{S^\prime}[g] - \hat{\E}_S[g] \big] \bigg] \\
&\leq \E_{S,S^\prime} \bigg[ \sup_{g \in \H} \hat{\E}_{\S^\prime}[g] - \hat{\E}_\S[g] \bigg] \\
&= \E_{S, S^\prime} \bigg[ \sup_{g \in \H} \frac{1}{m^2} \mathbf{1}^\top ( (\hat{\D^\prime} - \D^\prime) \otimes (\hat{\D^\prime} - \D^\prime) ) \mathbf{1} - \mathbf{1}^\top ( (\hat{\D} - \D) \otimes (\hat{\D} - \D) ) \mathbf{1} \bigg] \\
&= \E_{S, S^\prime} \bigg[ \sup_{h \in \H} \frac{1}{m^2} \sum_{ij} \Big( \big( \| h(x_i^\prime) - h(x_j^\prime) \|_2 - \D_{ij}^\prime \big)^2 - \big( \|h(x_i) - h(x_j) \|_2 - \D_{ij} \big)^2 \Big) \bigg] \\
&= \E_{\sigma,S, S^\prime} \bigg[ \sup_{h \in \H} \frac{1}{m^2} \sum_{ij} \sigma_{ij} \Big( \big( \| h(x_i^\prime) - h(x_j^\prime) \|_2 - \D_{ij}^\prime \big)^2 - \big( \|h(x_i) - h(x_j) \|_2 - \D_{ij} \big)^2 \Big) \bigg] \\
&\leq  \E_{\sigma, S^\prime} \bigg[ \sup_{h \in \H} \frac{1}{m^2} \sum_{ij} \sigma_{ij}  \big( \| h(x_i^\prime) - h(x_j^\prime) \|_2 - \D_{ij}^\prime \big)^2 \bigg] \\
&\qquad + \E_{\sigma, S}\bigg[ \sup_{h \in \H} \frac{1}{m^2} \sum_{ij} \sigma_{ij} \big( \|h(x_i) - h(x_j) \|_2 - \D_{ij} \big)^2 \bigg] \\
&= 2 \E_{\sigma,S} \bigg[ \sup_{h \in \H} \frac{1}{m^2} \sum_{ij} \sigma_{ij} \big( \|h(x_i) - h(x_j)\|_2 - \D_{ij} \big)^2 \bigg] \\
&=2 \Ra_m(\G)
\end{align*}

where the second equality holds by the i.i.d. assumption, the first inequality holds by Jensen's inequality, the second inequality holds by the sub-additivity of the supremum function, and we use the fact that $\E[\sigma_i] = \E[-\sigma_i]$.
\end{proof}

\begin{theorem}
\label{thm:rademacher}
Let $S = (x_1, x_2, ..., x_m)$ be a sample of data drawn i.i.d. from an arbitrary, but fixed distribution $D$. Let $\H = \{h : (\|h(x_i) - h(x_j)\|_2 - \D_{ij})^2 \leq M^2\}$ and $\G$ be the class of losses over $\H$. Then, with probability at least $1-\delta$, for all $h \in \H$

$$
R(h) \leq \hat{R}(h) + 2 \Ra_m(\G) + M \sqrt{\frac{2 \log \frac{1}{\delta}}{m}}
$$

\end{theorem}

\begin{proof}

We begin by using McDiarmid's inequality and Lemma \ref{lemma:martingale} to get the following upper bound:

With probability at least $1 - \delta$:

$$
|\varphi(S) - \E_S[ \varphi(S)] | \leq M \sqrt{ \frac{2 \log \frac{1}{\delta}}{m}}
$$

The bound follows immediately by Lemma \ref{lemma:rademacher}.
\end{proof}

Theorem \ref{thm:rademacher} gives us a Rademacher complexity bound on the generalization error for the distances between $m$ points. Notably, the bound depends on $M$, which is not known in general. However, if the domain of the output space $\Y$ is bounded, then $M$ can be trivially calculated. Below we give an upper bound on the Rademacher complexity for the case where $\H$ is linear.

\begin{corollary}
\label{corollary:linear}

Let $\H = \{ x \mapsto w \cdot x : \|w\| \leq \Lambda\}$, $\|x\| \leq r$, and $\D_{ij} \leq \beta$. Then, with probability at least $1-\delta$, for all $h \in \H$:

$$
R(h) \leq \hat{R}(h) + \frac{2 M^2}{m} + M \sqrt{\frac{2 \log \frac{1}{\delta}}{m}}
$$

where $M = \Lambda \max\{ 2r, \beta \}$.
\end{corollary}

\begin{proof}

\begin{align*}
\Ra_m(\G) &= \E_\sigma \bigg[ \sup_{h \in \H} \frac{1}{m^2} \sum_{ij} \sigma_{ij} \big( \|h(x_i) - h(x_j) \|_2 - \D_{ij} \big)^2 \bigg] \\
&= \frac{1}{m^2} \E_\sigma \bigg[ \sup_{\|w\| \leq \Lambda} \sum_{ij} \sigma_{ij} \big( \|w \cdot x_i - w \cdot x_j \|_2 - \D_{ij} \big)^2 \bigg] \\
&\leq \frac{\Lambda^2}{m^2} \E_\sigma \bigg[ \sum_{ij} \sigma_{ij} \big( \|x_i - x_j \|_2 - \D_{ij} \big)^2 \bigg] \\
&\leq \frac{\Lambda^2}{m^2} \bigg[ \E_\sigma \Big[ \sum_{ij} \sigma_{ij} \big(\|x_i - x_j\| - \D_{ij} \big)^2 \Big ]^2 \bigg]^{\frac{1}{2}} \\
&= \frac{\Lambda^2}{m^2} \bigg[ \E_\sigma \Big[ \sum_{ijkl} \sigma_{ij} \sigma_{kl} \big( \|x_i - x_j\| - \D_{ij} \big)^2 \big( \|x_k - x_l\| - \D_{ij} \big)^2 \Big] \bigg]^{\frac{1}{2}} \\
&= \frac{\Lambda^2}{m^2} \bigg[ \sum_{ij} \big( \|x_i - x_j\| - \D_{ij} \big)^4 \bigg]^{\frac{1}{2}} \\
&\leq \frac{\Lambda^2}{m^2} \bigg[ m^2 \max\{2r, \beta\}^4 \bigg]^{\frac{1}{2}} = \frac{\Lambda^2}{m} \max \{ 2r, \beta\} ^2
\end{align*}

Where the first inequality holds by the Cauchy-Schwarz inequality and supremum over $\|w\|$. The second inequality holds by Jensen's inequality. The expectation over $\sigma$ is eliminated because $\E_\sigma [\sigma_i \sigma_j] = 0$ whenever $i \neq j$. The final inequality holds by assumption and the non-negativity of distances. 
\end{proof}

Thus, we can guarantee that the squared distances will be close as long as the radius of the input and output distributions are small relative to the number of samples. Note that this bound also implies that the complexity grows quadratically with the norm of $w$.  More importantly, this bound can be computed directly from the data. A similar bound holds for the kernel formulation.

\begin{corollary}
\label{corollary:kernel}
Let $K : \X \times \X \to \R$ be a PDS kernel and let $\phi : \X \to \Hi$ be a feature mapping associated to $K$. Let $S \subseteq \{ x: K(x,x) \leq q^2\}$. Let $\H = \{x \mapsto w \cdot \phi(x): \|w\|_\Hi \leq \Lambda \}$. Then 

$$
\Ra_m(\G) \leq \frac{\Lambda^2}{m} \max \{\sqrt{2}q, \beta\}^2
$$
\end{corollary}

\begin{proof}
\begin{align*}
\Ra_m(\G) &= \frac{1}{m^2} \E_\sigma \bigg[\sup_{\|w\| \leq \Lambda} \sum_{ij} \sigma_{ij}\big (\| w \cdot \phi(x_i) - w \cdot \phi(x_j) \|_\Hi - \D_{ij} \big )^2 \bigg] \\
&\leq \frac{\Lambda^2}{m^2} \E_\sigma \bigg[ \sum_{ij} \sigma_{ij} \big( \|\phi(x_i) - \phi(x_j) \|_\Hi - \D_{ij} \big)^2 \bigg] \\
&\leq \frac{\Lambda^2}{m^2} \bigg[ \sum_{ij} \big( \| \phi(x_i) - \phi(x_j) \|_\Hi - \D_{ij} \big)^4 \bigg ]^{\frac{1}{2}} \\
&= \frac{\Lambda^2}{m^2} \bigg[ \sum_{ij} \Big( \sqrt{ \|\phi(x_i)\|_\Hi^2 + \|\phi(x_j)\|_\Hi^2 - 2\langle \phi(x_i) | \phi(x_j) \rangle_\Hi} - \D_{ij} \Big)^4 \bigg]^{\frac{1}{2}} \\
&\leq \frac{\Lambda^2}{m^2} \bigg[ \sum_{ij} \Big( \sqrt{ K(x_i, x_i) + K(x_j, x_j)} - \D_{ij} \Big)^4 \bigg]^{\frac{1}{2}} \leq \frac{\Lambda^2}{m} \max \{\sqrt{2}q, \beta\}^2
\end{align*}

where the first and second inequalities follow by the same logic in Corollary \ref{corollary:linear}. The second equality is a simple identity. The third inequality holds because $K$ is PDS and the final inequality holds by assumption.
\end{proof}

As we can see, the Rademacher complexity depends on the maximum of the square root of the diameter of the data in the Hilbert space defined by the kernel $K$ and the radius of the distance space. In some cases, this bound is actually tighter than the one in Corollary \ref{corollary:linear}. In particular, when $K$ is the RBF kernel, $q = 1$. Thus, in this case, the bound is independent of the dimensionality of $\X$.

\section{Conclusion} \label{sec:conclusion}

We presented a theoretical analysis of the similarity learning problem and gave probabilistic learning guarantees. By formulating the problem as a kernel regression from $\X$ to $\Y$, we were able to give explicit, data dependent bounds, based on the Rademacher complexity of the hypothesis class $\H$. We showed that, in the linear case, model complexity depended on the radius of the data in either $\X$ or $\Y$. Notably, we showed that when learning with kernels, the bound is independent of the dimensionality of $\X$.

\bibliography{mybib}
\bibliographystyle{plain}

\end{document}